\documentclass[letterpaper, 10 pt, conference]{ieeeconf}  

\IEEEoverridecommandlockouts                              

\usepackage{cite}
\usepackage{amsmath,amssymb,amsfonts}
\usepackage{mathtools}
\usepackage{multirow}
\usepackage{float}
\usepackage{algorithm}
\usepackage{algorithmic}
\usepackage{graphicx}
\usepackage{textcomp}
\usepackage{xcolor}
\usepackage{pifont}
\usepackage{cuted}
\usepackage[hidelinks]{hyperref}
\hypersetup{
  colorlinks,
  allcolors=.,
  urlcolor=blue,
}

\newtheorem{theorem}{Theorem}

\newtheorem{corollary}{Corollary}

\title{\LARGE \bf
CSR: Infinite-Horizon Real-Time Policies with \\Massive Cached State Representations \\
{\normalsize \textit{Extended Technical Report for Paper Accepted to IEEE RA-L}}
}

\author{Robin Karlsson$^{1*,2}$ and Go Suzui$^{1}$
\thanks{$^{1}$GODOT Inc, Hyogo, Japan}%
\thanks{$^{2}$Graduate School of Informatics, Nagoya University, Aichi, Japan.}%
\thanks{*Corresponding author: {\tt\small robin.karlsson@godot.inc}}%
\thanks{© 2026 IEEE. Personal use of this material is permitted. Permission from IEEE must be obtained for all other uses, in any current or future media, including reprinting/republishing this material for advertising or promotional purposes, creating new collective works, for resale or redistribution to servers or lists, or reuse of any copyrighted component of this work in other works.}
}

\begin{document}
\bstctlcite{IEEEexample:BSTcontrol}

\maketitle
\thispagestyle{empty}
\pagestyle{empty}

\begin{abstract}
Deploying massive large language models (LLMs) as continuous cognitive engines for robotics is bottlenecked by the time-to-first-token (TTFT) latency required to process extensive state histories.
Existing solutions like RAG or sliding windows compromise global context or incur prohibitive re-computation costs.
We formalize the optimal task structure for minimizing latency and theoretically prove that prefix stability, incremental extensibility, and asynchronous state reconciliation are necessary conditions for real-time performance.
Building on these proofs, we introduce the Cached State Representation (CSR) framework as the practical instantiation of these properties, ensuring optimal KV-cache reuse.
To sustain these properties over infinite horizons, we further propose an Asynchronous State Reconciliation (ASR) algorithm that offloads state memory eviction to a parallel computational resource to eliminate latency spikes.
On a physical robot wirelessly connected to an on-premise GPU server, CSR achieves a 26-fold latency reduction (14.67s to 0.56s) for 120K token contexts with a 235B parameter model compared to a standard baseline.
On an embodied AI benchmark, we achieve SOTA recall (0.836 vs. 0.459) while maintaining RAG-level latency.
ASR is validated to sustain bounded, spike-free TTFT over 10 eviction cycles in continuous real-world operation.
Together, CSR and ASR enable massive LLMs to function as continuously operating, high-frequency ($\ge$ 2 Hz) embodied policies.
\end{abstract}

\section{INTRODUCTION}

Large Language Models (LLMs) have emerged as the dominant paradigm for realizing general-purpose robots, providing agents with commonsense reasoning~\cite{ahn2022saycan, huang2022innermonologue, liang2023codeaspolicies} and open-world generalization capabilities~\cite{driess2023palme, wang2023voyager, zitkovich2023rt2}.
To fully leverage these cognitive engines, robots require an extensive state representation that encapsulates static embodiment definitions, continuously updating sensory observations, and accumulated episodic history~\cite{packer2023memgpt, xiao2023streamingllm}.
However, maintaining such high-fidelity context introduces a fundamental conflict with real-time performance.
As the state history grows, the computational cost of processing input sequences grows linearly or quadratically, resulting in prohibitively high time-to-first-token (TTFT) latency~\cite{liu2024cachegen}.
This computational bottleneck typically forces a trade-off between the depth of the agent's context and its ability to maintain the high-frequency decision-making loops required for responsive interaction~\cite{zitkovich2023rt2, ghosh2024octo, kim2024openvla}.

While key-value (KV) caching effectively accelerates text generation in static contexts~\cite{kwon2023vllm, gim2024promptcache}, it is ill-suited for the dynamic state updates inherent to robotics.
Standard inference engines operate optimally on append-only sequences.
However, robotic high-level policies require continuous state evolution over long time horizons, where new observations are integrated and obsolete data must eventually be evicted due to finite memory.
Naively modifying the prompt structure by altering tokens that precede the end of the sequence invalidates the cached attention computations for all subsequent tokens.
Consequently, the model is forced to reprocess the entire context history for every inference step, causing latency to scale with context length.
Current mitigation strategies, such as Retrieval-Augmented Generation (RAG)~\cite{packer2023memgpt, wang2023longmem} or aggressive context pruning~\cite{xiao2023streamingllm, wang2023deps}, circumvent this cost but sacrifice global context awareness, limiting the agent's ability to execute consistent, long-horizon behavior.

\begin{figure}[t!]
\centering
\includegraphics[width=0.48\textwidth]{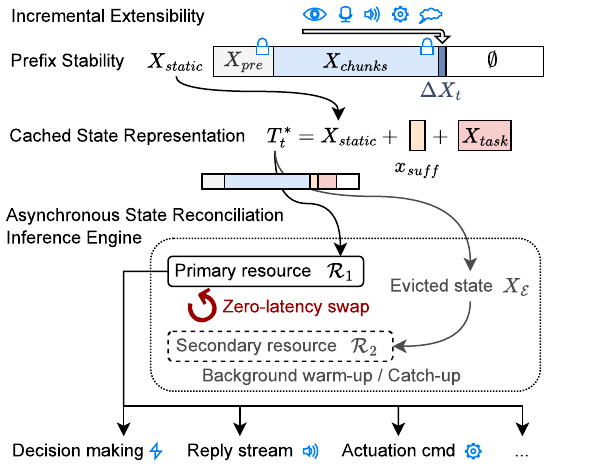}
\vspace{-4mm}
\caption{Overview of the Cached State Representation (CSR) framework for real-time robotic control with massive LLMs. (Top) CSR implements the optimal task structuring strategy $S^*$ by isolating volatile state information into a transient dynamic suffix $x_{suff} + X_{task}$ and appending state information chunks $x_t$ to a static prefix $X_{static}$. This structure satisfies \textit{Prefix Stability} and \textit{Incremental Extensibility}. (Bottom) The Asynchronous State Reconciliation (ASR) inference engine enables infinite-horizon continuous operation by eliminating eviction latency spikes. A secondary resource $\mathcal{R}_2$ precomputes the KV-cache of the evicted state $X_\mathcal{E}$ and synchronizes with the primary resource $\mathcal{R}_1$ via a zero-latency atomic swap.}
\label{fig:front_fig}
\vspace{-6mm}
\end{figure}

To resolve this conflict, we first formalize the optimal task structure $S^*$ for minimizing latency under hardware and information constraints. We theoretically prove two necessary conditions: prefix stability, requiring that historical context remains immutable to preserve cache validity, and incremental extensibility, requiring that new information be strictly appended. Building on these proofs, we introduce the Cached State Representation (CSR) framework as the practical instantiation of $S^*$, ensuring optimal KV-cache reuse. To sustain these properties over infinite horizons without violating memory limits, we further propose an Asynchronous State Reconciliation (ASR) algorithm that decouples the expensive state eviction process from the primary inference loop. ASR utilizes parallel hardware to precompute the KV-cache for the new state, eliminating latency spikes.

Experiments with a 235B parameter model demonstrate a 26-fold latency reduction (14.67\,s to 0.56\,s) for 120K token contexts, with ASR validated to sustain bounded, spike-free operation in real-world deployments. Together, CSR and ASR enable massive LLMs to function as continuously operating, high-frequency ($\ge$ 2\,Hz) embodied policies without restrictive state memory size.
Our specific contributions are:
\begin{itemize}
\item A theoretical formalization of the optimal task structure $S^*$, establishing prefix stability and incremental extensibility as necessary conditions for satisfying real-time deadlines $\tau_{max}$.
\item The Cached State Representation (CSR) framework as the practical implementation of $S^*$, utilizing Asynchronous State Reconciliation (ASR) to enable infinite-horizon real-time operation.
\item Empirical validation across varying model sizes (30B--235B) and hardware configurations, demonstrating high-frequency responsiveness ($\ge$ 2\,Hz) in continuous real-world deployments and SOTA embodied AI task performance recall (0.836 vs.\ 0.459) at RAG-level latency, proving both the physical viability and task performance advantage of our framework.
\end{itemize}

The rest of the paper is organized as follows. Sec.~\ref{sec:related_works} lists related works. Sec.~\ref{sec:csr} presents our formalization and method. Sec.~\ref{sec:experiments} and Sec.~\ref{sec:results} explain experiments and results for evaluating performance. Sec.~\ref{sec:conclusions} positions our findings.

\section{RELATED WORKS}
\label{sec:related_works}

\noindent \textit{LLM applications in robotics:}
LLMs provides an unprecedented possibility to leverage commonsense knowledge in various robotics tasks.
Using LLMs for autoregressive action decision making~\cite{ahn2022saycan, yao2023react, shinn2023reflexion} effectively enhances the versatility and robustness of real-world robots with limited feasibility to learn by trial-and-error.
High-level planing~\cite{rana2023sayplan, nottingham2023deckard, liu2023llmp} benefits from the remarkable reasoning capability of LLMs.
Multimodal reasoning LLMs~\cite{zeng2023socratic, driess2023palme} and Vision-Language-Action (VLA) models~\cite{zitkovich2023rt2, padalkar2023rtx, kim2024openvla, ghosh2024octo} enables robots to intelligently assess the robot's state and predict optimal actions directly from perceptual observations.
LLMs can also serve as an internal reward model that allows robots to improve task performance from real-world environment feedback~\cite{huang2022innermonologue, wang2023voyager, wang2023deps, zhang2023rememberer}.
Coding LLMs can also be used to iteratively generate highly performant programmatic reward functions~\cite{ma2024eureka, ma2024dreureka} and action execution code~\cite{singh2023progprompt, liang2023codeaspolicies}.
Our proposed Cached State Representation (CSR) expands the application domain of massive LLMs in robotics by enabling high frequency action decision making and instantaneous interaction with other agents including humans.

\noindent \textit{Agentic state representation:}
A fundamental challenge in LLM-based robotics is how to represent a continuous, 4D spatiotemporal reality into the discrete, sequential latent space required for Transformer architectures.
The 3D scene graph approach abstracts geometry into hierarchical semantic nodes~\cite{rana2023sayplan, mandi2024roco, gu2024conceptgraphs}. By linearizing these graphs into structured text, LLMs gain a symbolic, queryable summary of the environment.
Semantic topological maps represent the environment as searchable, graph-based manifolds where nodes or navigable landmarks consist of localized semantic embeddings~\cite{shah2022lmnav, huang2023vlmap}.
The structured state approach bypass symbolic labeling in favor of directly encode environment information as token sequences with explicit metric and temporal information for fine-grained spatial reasoning.
Metric tokenization methods~\cite{chen2024spatialvlm, zitkovich2023rt2, mariano2024kat} discreteizes 3D geometry. 
Multimodal embedding stream methods~\cite{huang2023voxposer, wang2025memento} creates queryable latent embedding state buffers from encoded high-dimensional sensory percepts.
%
%
Short-term memory caching approaches attempt to maintain long-horizon agentic consistency within a finite context window. RecurrentGPT~\cite{zhou2023recurrentgpt} implements a recurrent mechanism where the model generates an latent state token embedding. MemGPT~\cite{packer2023memgpt} utilizes a paging mechanism to swap relevant state information in and out of the token sequence. Closest to our method, StreamingLLM~\cite{xiao2023streamingllm} enables infinite-length generation by recomputing the KV cache over a fixed-size sliding window of recent tokens and attention sinks, trading long-horizon memory fidelity for low latency through a short window rather than cache reuse.
%
%
Long-term episodic memory frameworks utilize external vector databases and retrieval-augmented generation (RAG) to bridge experiences across multiple trials~\cite{park2023generativeagents, wang2023longmem, zhao2026amabench}. While effective for lifelong learning, these retrieval-based approaches often introduce significant query latency and context fragmentation, making them ill-suited for real-time behavior generation.
Our CSR framework addresses the memory fidelity limitation of StreamingLLM~\cite{xiao2023streamingllm} by replacing widow truncation with structured state partitioning, enabling KV cache reuse while retaining the full state history up to $\mathcal{O}(10^5)$ tokens, with latency scaling linearly in sequence length.

\noindent \textit{Real-Time and efficient LLM inference:}
Continuously operating robots performing tasks relying on real-time interaction with the environment requires low latency and high throughput LLM inference engines.
Model compression techniques, such as weight quantization~\cite{frantar2023gptq, lin2024awq} and low-precision architectural paradigms~\cite{wang2025bitnet}, reduce the memory and computational overhead of LLMs with limited performance degradation.
Speculative decoding and parallel sampling strategies~\cite{leviathan2023fastspeculative, cai2024medusameth} speeds up generation by using small draft models to predict multiple tokens each forward pass.
Innovations in efficient attention computation mitigates the quadratic memory complexity limiting very long states by hardware-optimized kernel~\cite{dao2024flashattention2} and distributed computation strategies~\cite{liu2024ringattention}.
%
Recent LLM serving frameworks implement KV-cache optimizations including dynamic memory orchestration via paged allocation~\cite{kwon2023vllm} and continuous batching~\cite{yu2022orca} to eliminate VRAM fragmentation and maximize inference throughput.
Semantic context reuse strategies like RadixAttention~\cite{zheng2023sglang} and prefix caching~\cite{gim2024promptcache} allow for the instantaneous reuse of pre-computed token sequences.
KV-Cache compression frameworks~\cite{liu2024cachegen, pham2024chunkattention} minimize the memory footprint and transmission latency of state information shared over a network.
Our proposed CSR framework extends these serving primitives by providing a theoretical framework for optimal state structuring for KV-cache optimization frameworks like vLLM~\cite{kwon2023vllm} to ensure minimal latency, and an ASR algorithm that together with CSR ensures prefix stability and zero-latency state eviction for continuous, high-frequency robotic control loops. While CSR does not reduce continuous token output generation speed, CSR is compatible with and compliments existing methods for this purpose~\cite{leviathan2023fastspeculative}.

\section{CACHED STATE REPRESENTATION}
\label{sec:csr}

In this section we present how intermediate computational states are cached and reused to speed up single Transformer-based LLM inference tasks.
Thereafter we show how to extend cache reuse to sequential requests via the Cached State Representation (CSR) framework, and how the Asynchronous State Reconciliation (ASR) algorithm eliminates eviction-induced latency spikes to sustain this property over infinite horizons, achieving minimal \textit{time-to-first-token} (TTFT) for high-frequency robot AI inference. 

\subsection{Background: Transformer Architecture and KV Caching}

The Transformers architecture is a general-purpose model of computation that excels in learning the probability distribution of sequences of discrete elements
\begin{equation}
    X = (x_1, \ldots , x_N) \in \mathcal{X}^N
\end{equation}
by autoregressive decomposition of the joint probability of a sequence:
\begin{align}
    P(X) &= P(x_1, \ldots, x_N) \\
         &= P(x_n | x_1, \ldots, x_{n-1}) \ldots P(x_2|x_1) P(x_1).
\end{align}

Current LLMs are variations of the decoder-only Transformer architecture and model token sequence probabilities. A token  represent a substring $x$ from the finite set $\mathcal{X}$ containing $K$ substrings from which all possible strings in a language can be created. An LLM predicts a discrete probability distribution $y \in \mathbb{R}^{K}$ for the next token $x_{n+1}$ by first mapping the discrete token sequence $X$ into a latent token embedding matrix $H^{(0)} = [h_1^{(0)}, \ldots, h_n^{(0)}]$ with embedding vectors $h \in \mathbb{R}^{d_h}$. Next, $L$ sequential decoder blocks recurrently updates $H^{(\ell-1)}$ to $H^{(\ell)}$ by a self-attention module and feed forward networks (FFN):
\begin{align}
    k_n^{(\ell)} &= W_K^{(\ell)} h_n^{(\ell-1)} \in \mathbb{R}^{1 \times d_k} \label{eq:k} \\
    v_n^{(\ell)} &= W_V^{(\ell)} h_n^{(\ell-1)} \in \mathbb{R}^{1 \times d_v} \label{eq:v} \\
    q_n^{(\ell)} &= W_Q^{(\ell)} h_n^{(\ell-1)} \in \mathbb{R}^{1 \times d_k} \\
    a_n^{(\ell)} &= \text{softmax}\left(\left[\frac{q_n^{(\ell)} \cdot k_{\tau}^{(\ell)}}{\sqrt{d_k}}\right]_{\tau=1}^{n}\right) \label{eq:causal_attention_1} \\
    &= \text{softmax}\left(\frac{q_n^{(\ell)} [\mathcal{K}_{1:n}^{(\ell)}]^T}{\sqrt{d_k}}\right) \in \mathbb{R}^{n} \label{eq:causal_attention_2} \\
    o_n^{(\ell)} &= \sum_{\tau=1}^n a_{n,\tau}^{(\ell)} v_{n,\tau}^{(\ell)} = a_n^{(\ell)} [\mathcal{V}_{1:n}^{(\ell)}] \in \mathbb{R}^{d_v} \label{eq:causal_attention_3} \\
    z_n^{(\ell)} &= \text{LN}(h_n^{(\ell-1)} + W_O^{(\ell)} o_n^{(\ell)}) \\
    h_n^{(\ell)} &= \text{LN}\left(z_n^{(\ell)} + \text{FFN}^{(\ell)}\left(z_n^{(\ell)}\right)\right) \\
    y &= \text{softmax}(W_Y H^{(L)}) \\
    x_{n+1} &= f(y) \label{eq:next_token_pred}
\end{align}
where
\begin{align}
[\mathcal{K}_{1:n}^{(\ell)}] &= [k_1^{(\ell)}, \ldots,  k_n^{(\ell)}] \label{eq:K_cache} \\
[\mathcal{V}_{1:n}^{(\ell)}] &= [v_1^{(\ell)}, \ldots,  v_n^{(\ell)}] \label{eq:V_cache}
\end{align} 
denotes cached concatenations of keys and value vectors for tokens 1 to $n$. Caching key and value vectors as in \eqref{eq:K_cache}-\eqref{eq:V_cache} effectively amortizes or reuses the computation of intermediate computational states for tokens 1 to $n$, as those states are not altered by subsequent tokens due to the causal attention mechanism.

Utilizing KV caching is a standard technique for reducing latency when processing an independent growing token sequence $X$ representing a single inference task $T$. In this paper, we consider robot AI inference tasks such as decision making, reply, and actuation command generation. However, high-frequency embodied robot AI tasks are not independent, but sequential and conditional on a subset of all available information in the robot system.

In the next section, we explain why naively structured sequential inference tasks generally breaks the cache, resulting in high TTFT latency for robot AI applications.

\subsection{KV Cache Reuse Across Sequential Tasks}
\label{sec:kv_cache_reuse}

We formalize a robot AI inference task $T$ as
\begin{equation}
    \label{eq:task_repr}
    T = (X_{in}, X_{out}) = [x_1^{in}, \ldots, x_N^{in}, x_{N+1}^{out}, \ldots, x_M^{out}],
\end{equation}
where $X_{in}$ and $X_{out}$ are the input and output token sequences. $X_{in}$ contains the task specification and information required to complete the task. $X_{out}$ is initially an empty set $\emptyset$ and is appended with new tokens generated through \eqref{eq:next_token_pred}.

When starting to process a new task
\begin{equation}
    T' = (X_{in}', \emptyset) = [\acute{x}_1^{in}, \ldots, \acute{x}_S^{in}],
\end{equation}
the cached KV cache vectors \eqref{eq:K_cache}-\eqref{eq:V_cache} for all tokens in \eqref{eq:task_repr} becomes invalid at the first token in $T'$ differing with the corresponding token in $T$:
\begin{equation}
    \label{eq:first_differing_token}
    i^* = \text{min}\{i \in \{1, \ldots, S\} : T_i \neq T_i'\}.
\end{equation}
In particular, if the first token $i = 1$ in \eqref{eq:first_differing_token} for $T'$ is different from $T$, the entire KV cache for $X_{in}'$ needs to be recomputed from scratch. Note that this behavior is expected and unavoidable for general independent tasks $T$ and $T'$.

In the next section we formalize the sequential robot AI inference task objective, and derive required properties for an optimal task representation.



\subsection{Optimal Task Structure}
\label{sec:optimal_task_structure}

To derive an optimal solution to the problem explained in Sec.~\ref{sec:kv_cache_reuse}, we first formalize the optimization objective for sequential robot AI inference tasks as
\begin{equation}
\label{eq:seq_robot_ai_optimization_obj}
\min_{S} \max_{t} \text{TTFT}(T_t | S).
\end{equation}

The objective \eqref{eq:seq_robot_ai_optimization_obj} is subject to the following constraints:
\begin{align}
    \text{TTFT}(T_t | S) &\leq \tau_{max} \quad \forall t \label{eq:real-time_constraint} \\
    \text{Info}(T_t | S) &\supseteq \mathcal{I}_{req}^{(t)} \quad \forall t \label{eq:information_sufficiency_constraint} \\
    \text{Memory}(T_t | S) &\leq M_{max} \quad \forall t \label{eq:finite_memory_constraint}
\end{align}
where
$S$ is a task structuring strategy from the set of all possible strategies $\mathcal{S}$.
$\tau_{max}$ is maximum allowed deadline for real-time inference.  
$\text{Info}(T_t)$ is the information contained in the task representation $T_t$ and must contain all information required to complete the task $\mathcal{I}_{req}^{(t)}$.
$M_{max}$ represents the finite memory in which $T_t$ can be stored. 

Next we derive TTFT as a function of cache reuse. From equations \eqref{eq:k}-\eqref{eq:next_token_pred}, the computational cost for processing task $T_t$ is
\begin{equation}
    \label{eq:TTFT_derivation_1}
    \text{TTFT}(T_t) = \sum_{i=i^*}^{|T_t|} \sum_{\ell=1}^{L} C \cdot f_{ops}(i, \ell),
\end{equation}
where $C$ is the computational cost per operation and $f_{ops}(i, \ell)$ is the number of operations for token $i$ at layer $\ell$. For standard Transformers utilizing KV caching, the number of operations is a linear function of latent embedding dimensionality $d_h$ and the number of preceding tokens due to the causal attention mechanism \eqref{eq:causal_attention_1}-\eqref{eq:causal_attention_3}:
\begin{equation}
    \label{eq:f_ops}
    f_{ops}(i, \ell) = \mathcal{O}(i \cdot d_h).
\end{equation}
Therefore \eqref{eq:TTFT_derivation_1} can be simplified as a linear function of $i$:
\begin{equation}
    \label{eq:TTFT_derivation_2}
    \text{TTFT}(T_t) \propto L \: d_h \: \sum_{i=i^*}^{|T_t|} i = L \: d_h \: \frac{(|T_t| - i^* + 1)(|T_t| + i^*)}{2}.
\end{equation}
It follows that minimizing TTFT requires maximizing $i^*(T_{t-1}, T_t)$ for sequential tasks $T_t$ and $T_{t+1}$ according to \eqref{eq:first_differing_token}.


\begin{theorem}\label{theorem:prefix_stability}[Prefix Stability]
To satisfy the real-time constraint \eqref{eq:real-time_constraint} for sequential tasks with bounded computation, sufficient prefix stability is necessary.
The task structure $S$ must ensure that the first differing token index $i^*$ is greater than or equal to a minimum stable prefix length $i^*_{min}$:
\begin{equation}
    \label{eq:prefix_stability_condition}
    i^*(T_t, T_{t+1}) \ge i^*_{min} \quad \forall t,
\end{equation}
where $i^*_{min}$ is defined as the minimum prefix match $i^*$ such that the computational cost $N_{FLOPs}$ for the new task $T_{t+1}$ satisfies the real-time budget $B_{max}$:
\begin{equation}
    N_{FLOPs}(T_{t+1} | i^* = i^*_{min}) \le B_{max}.
\end{equation}
An optimal strategy $S^*$ minimizes~\eqref{eq:seq_robot_ai_optimization_obj} by structuring sequential tasks $T_t$ to achieve optimal prefix stability:
\begin{equation}
    \label{eq:optimal_strategy_definition}
    S^* \implies i^*(T_t, T_{t+1}) = |T_t| + 1 \quad \forall t.
\end{equation}

\end{theorem}

\begin{proof}
    See Appendix A. 
\end{proof}

The proof is based on estimating the required FLOPs without prefix stability to be $\mathcal{O}(100)$ higher than with prefix stability.

\begin{theorem}[Incremental Extensibility]
\label{theorem:incremental_extensibility}
For a robot AI system where the total information context is non-decreasing ($\text{Info}(T_{t+1} | S) \supseteq \text{Info}(T_t | S)$), any task structuring strategy $S$ must be incrementally extensible by a state increment $\Delta X_t$
\begin{equation}
    \label{eq:incremental_extensibility}
    T_{t+1} = T_t \oplus \Delta X_t
\end{equation}
in order to satisfy the real-time constraint \eqref{eq:real-time_constraint} and information sufficiency constraint \eqref{eq:information_sufficiency_constraint} for all tasks in a sequence $(T_t)$.
\end{theorem}

\begin{proof}
    See Appendix B. 
\end{proof}

The proof shows by contradiction that any $S$ not following Theorem~\ref{theorem:incremental_extensibility} violates Theorem~\ref{theorem:prefix_stability}.

\begin{theorem}\label{theorem:eviction}[Eviction]
Given the incremental extensibility from Theorem~\ref{theorem:incremental_extensibility}, any strategy $S$ for infinite-horizon operation ($t \to \infty$) must include an eviction operation $\mathcal{E}$ to ensures the memory constraint \eqref{eq:finite_memory_constraint} is satisfied:
\begin{equation}
\label{eq:eviction_operation}
T_{t+1} = \mathcal{E}(T_t \oplus \Delta X_t)
\end{equation}
\end{theorem}

\begin{proof}
    See Appendix C. 
\end{proof}

The proof by contradiction demonstrates that the unbounded growth of $T_t$ as $t \rightarrow \infty$ according to Theorem~\ref{theorem:incremental_extensibility} and inevitably violates \eqref{eq:finite_memory_constraint}.

Theorems~\ref{theorem:prefix_stability}~and~\ref{theorem:eviction} creates a fundamental tension, as applying the eviction operation \eqref{eq:eviction_operation} violates token sequence order required for prefix stability \eqref{eq:prefix_stability_condition}.
Theorem~\ref{theorem:reconciliation} proposes a solution to satisfy Theorems~\ref{theorem:prefix_stability}-\ref{theorem:eviction} simultaneously by performing asynchronous state reconciliation in parallel on a secondary computation hardware.

\begin{theorem}\label{theorem:reconciliation}[Asynchronous State Reconciliation]
To satisfy the real-time constraint \eqref{eq:real-time_constraint} (Theorem~\ref{theorem:prefix_stability}) during an eviction operation \eqref{eq:eviction_operation} (Theorem~\ref{theorem:eviction}), the KV cache computation for the new evicted state $T_t^{\mathcal{E}}$ must be offloaded to a parallel computational resource.

Continuous real-time operation is achieved by:
\begin{enumerate}
    \item Warming up $T_t^{\mathcal{E}}$ on the secondary resource.
    \item Reconciling $T_t^{\mathcal{E}}$ with the new token sequence $\Delta X_t$ processed by the primary resource during warm-up.
    \item Swapping the primary state $T_t$ with the reconciled state $T_t^{\mathcal{E}}$ before $T_t$ violates the memory constraint \eqref{eq:finite_memory_constraint}.
\end{enumerate}
\end{theorem}

\begin{proof}
    See Appendix D. 
\end{proof}

The proof is based on showing the eviction process must be offloaded to satisfy Theorem~\ref{theorem:prefix_stability}, and construct a method and derive an inequality constraining viable free variables.

\begin{corollary}\label{corollary:optimality_of_csr}[Optimal Strategy]
A task structure strategy $S$ that satisfies Theorem~\ref{theorem:prefix_stability}-\ref{theorem:reconciliation} is the optimal strategy $S^*$ for optimizing~\eqref{eq:seq_robot_ai_optimization_obj} while satisfying the constraints~\eqref{eq:real-time_constraint}-\eqref{eq:finite_memory_constraint}.
\end{corollary}

\begin{proof}\label{proof:corollary_1}
    See Appendix E. 
\end{proof}

The proof is based on mapping the optimization of~\eqref{eq:seq_robot_ai_optimization_obj} and constraints satisfaction~\eqref{eq:real-time_constraint}-\eqref{eq:finite_memory_constraint} to each theorem.

In the next section we present the Cached State Representation (CSR) framework as the practical task representation structure strategy that is closest to $S^*$.

\subsection{Cached State Representation (CSR) Framework}
\label{sec:csr_framework}

In this section we introduce our Cached State Representation (CSR) framework as the concrete task representation implementing the optimal strategy $S^*$ derived in Sec.~\ref{sec:optimal_task_structure}. We derive CSR as the direct embodiment of the Prefix Stability (Theorem~\ref{theorem:prefix_stability}) and Incremental Extensibility (Theorem~\ref{theorem:incremental_extensibility}) properties of $S^*$:

\begin{figure}[!t]
\centering
\includegraphics[width=0.48\textwidth]{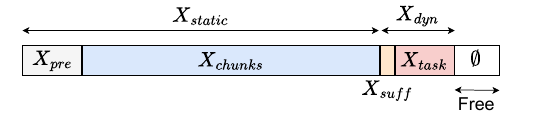}
\vspace{-8mm}
\caption{Structure of the Cached State Representation (CSR). The context is partitioned into a reusable static prefix $X_{static}$ and a transient dynamic suffix $X_{dyn}$.}
\label{fig:csr_structure}
\vspace{4mm} 
\includegraphics[width=0.48\textwidth]{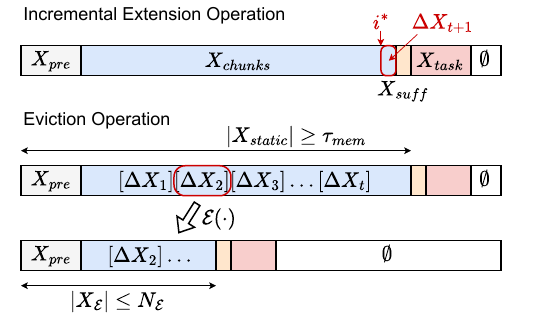}
\vspace{-8mm}
\caption{CSR operations. (Top) Incremental Extension appends $x_{t+1}$ while maintaining cache. (Bottom) Eviction removes older chunks to satisfy memory constraints.}
\label{fig:csr_operations}
\vspace{-6mm}
\end{figure}

First, a task structure $T_t$ that satisfies Theorem~\ref{theorem:prefix_stability} must guarantee a minimal stable prefix $i^*_{min}$ in~\eqref{eq:prefix_stability_condition}.
We achieve this by partitioning the robot AI's state into a static state prefix $X_{static}$ containing immutable information. 
Static information include robot profile, agency and behavior information, in-context learning information, action descriptions, etc. We call this set of information the state prefix $X_{pre}$.
Additionally, all information accumulated during robot AI operation including observational experience, internal thoughts, and action execution output, etc,. are stored as state chunks $x_t$ in an ordered sequence $X_{chunks} = (x_0, \ldots, x_t)$.
The KV cache for
\begin{equation}
    \label{eq:x_static}
    X_{static} = (X_{pre}, X_{chunks})
\end{equation}
is computed once and guarantees $i^* \ge |X_{static}|$.

Secondly, to satisfy Theorem~\ref{theorem:incremental_extensibility}, $T_t$ must be incrementally extensible~\eqref{eq:incremental_extensibility}. We implement this by appending new state information $x_t$ as a sequence of state chunks 
which accumulates all information the robot AI obtains during operation such as observational experience, internal thoughts, and action execution output, etc.
Continuously appending $x_t$ while performing high-frequency robot AI inference tasks allows the KV cache computations~\eqref{eq:k}-\eqref{eq:v} to be amortized across many small state chunk subsets $\Delta X_{chunks} = (x_{t+1}, \ldots, x_T)$ so that $i^*$ in \eqref{eq:first_differing_token} is close to the optimal prefix stability~\eqref{eq:optimal_strategy_definition}.

Practical general-purpose robot AI systems also need to consider dynamic information $X_{dyn}$ representing non-sequential information:
First, an information component where to write instantaneous information lacking historical value, such as current time, to avoid cluttering $X_{chunks}$. We call this component the state suffix $X_{suffix}$.
Secondly, a component $X_{task}$ where to specify the current task instruction, such as decision making or reply generation.
Due to its dynamic nature, the KV cache for
\begin{equation}
    \label{eq:x_dynamic}
    X_{dyn} = (X_{suff}, X_{task})    
\end{equation}
is generally recomputed for every task and should be as small as possible. 

The resulting CSR structure $T_t^*$ is illustrated in Fig.~\ref{fig:csr_structure} and formalized as the concatenation of all aforementioned information components:
\begin{equation}
    T_t^* = (X_{static}, X_{dyn}) \\
\end{equation}
\begin{equation}
    |X_{static}| \gg |X_{dyn}|
\end{equation}
\begin{equation}
    \label{eq:csr_i_star}
    i^* = |X_{static}| + 1.
\end{equation}
Figure~\ref{fig:csr_operations} illustrate the incremental extension and eviction operations on $T_t^*$.
Note that $T_t^*$ approaches $S^*$ as $|X_{dyn}| \rightarrow 0$.
In the next section we present a practical implementation of the reconciliation method based on Theorem~\ref{theorem:reconciliation}.





\subsection{Asynchronous State Reconciliation Algorithm}

Theorem~\ref{theorem:reconciliation} necessitates a state reconciliation method to overcome the fundamental tension between Theorems~\ref{theorem:prefix_stability} (stable prefix) and~\ref{theorem:eviction} (eviction).
We present the asynchronous state reconciliation method in Algorithm~\ref{alg:asynchronous_reconciliation}, with the process flow illustrated in Fig.~\ref{fig:reconciliation}. This implementation realizes Theorem~\ref{theorem:reconciliation} and follows the constructive proof provided in Appendix~D.

The method bifurcates the inference process into two concurrent execution tracks: an active primary resource $\mathcal{R}_1$ responsible for serving robot AI inference tasks, and a secondary resource $\mathcal{R}_2$ used for asynchronously reconciling the evicted state $X_\mathcal{E}$ in the background.
When $\mathcal{R}_1$ reaches a memory threshold $\tau_{mem}$, it triggers the eviction operation \eqref{eq:eviction_operation} and begins recomputing the KV cache \eqref{eq:K_cache}-\eqref{eq:V_cache} (i.e., warmup) on $\mathcal{R}_2$ (Lines 10-14). During this phase, $\mathcal{R}_1$ continues to serve inference tasks in parallel, while new incoming state chunks $\Delta X_t$ are simultaneously buffered in $X_{buffer}$ (Lines 6-8).
Once the initial KV cache is recomputed, $\mathcal{R}_2$ iteratively updates the cache with the buffered chunks (i.e., catch-up) until the buffer is drained (Lines 16-20).
Upon convergence, an atomic swap promotes $\mathcal{R}_2$ to the primary resource (Lines 22-23). When $\mathcal{R}_2$ eventually fills, the reconciliation process is repeated on $\mathcal{R}_1$.
The method ensures continuous real-time operation without latency spikes.

\section{EXPERIMENTS}
\label{sec:experiments}

We evaluate the proposed CSR framework on four distinct objectives: (1) quantifying the Time-to-First-Token (TTFT) reduction, (2) assessing the hardware generalization, (3) validating the stability of the ASR algorithm, and (4) benchmark embodied AI task performance.

\subsection{Experimental Setup}
\noindent \textit{Model and hardware:} Unless otherwise noted, experiments utilize the \textit{Qwen 2.5-235B-Instruct-FP8} model served via a vLLM inference server with all caching features enabled on 4$\times$H100 GPUs. We utilize \textit{Qwen 2.5-32B-Instruct} for single GPU ablation studies.
To assess real-world performance, we deploy the CSR framework on a physical robot. The system is built in ROS 2 Jazzy and runs on a Raspberry Pi 5.
The robot is connected to an on-premise GPU server by a wireless network, reflecting the emerging mobile edge computing paradigm enabled by post-5G infrastructure~\cite{nedo2024post5g}.

\begin{algorithm}[t!]
\caption{Asynchronous State Reconciliation}
\label{alg:asynchronous_reconciliation}
\begin{algorithmic}[1]
\REQUIRE Stream of state chunks $\{\Delta X_t\}$
\REQUIRE Parameters: $\tau_{\text{mem}}$, $N_{\text{catchup}}$
\item[\textbf{State Variables:}]
\item[] $X_{\text{static}} \gets X_{\text{pre}}$ \hfill // \textit{Current static state prefix}
\item[] $\Delta X_{\text{buffer}} \gets \emptyset$ \hfill // \textit{State chunk buffer}
\item[] $X_{\varepsilon} \gets \emptyset$ \hfill // \textit{Evicted/compressed state}
\item[] $r_{\text{primary}} \gets \mathcal{R}_1$ \hfill // \textit{Currently active resource}
\item[] $r_{\text{secondary}} \gets \mathcal{R}_2$ \hfill // \textit{Resource being warmed up}
\item[] $\textit{is\_reconciling} \gets \texttt{false}$ \hfill // \textit{Status flag}
\STATE \textbf{Procedure} \textsc{Increment}($\Delta X_t$):
    \STATE // \textit{Always update the active primary state}
    \STATE $X_{\text{static}} \gets X_{\text{static}} \oplus \Delta X_t$
    \STATE
    \STATE // \textit{Fork data if a background reconciliation is active}
    \IF{$\textit{is\_reconciling}$}
        \STATE $\Delta X_{\text{buffer}} \gets \Delta X_{\text{buffer}} \oplus \Delta X_t$
    \ENDIF
    \STATE
    \IF{$|X_{\text{static}}| \geq \tau_{\text{mem}}$ \AND \NOT $\textit{is\_reconciling}$}
        \STATE $X_{\varepsilon} \gets \text{Evict}(X_{\text{static}})$
        \STATE Launch \textsc{Reconcile}() in parallel // \textit{Asynchronous}
        \STATE $\textit{is\_reconciling} \gets \texttt{true}$
    \ENDIF
\STATE
\STATE \textbf{Procedure} \textsc{Reconcile}():
    \STATE // \textit{Warm up secondary resource with evicted state}
    \STATE $\text{LLM}(X_{\varepsilon}, r_{\text{secondary}})$
    \STATE
    \STATE // \textit{Reconcile buffered chunks until sufficiently caught up}
    \STATE // \textit{NOTE: Reminder is computed at next task inference}
    \WHILE{$|\Delta X_{\text{buffer}}| > N_{\text{catchup}}$}
        \STATE $X_{\varepsilon} \gets X_{\varepsilon} \oplus \Delta X_{\text{buffer}}$
        \STATE $\Delta X_{\text{buffer}} \gets \emptyset$
        \STATE $\text{LLM}(X_{\varepsilon}, r_{\text{secondary}})$
    \ENDWHILE
    \STATE
    \STATE // \textit{Swap primary and secondary resource pointers}
    \STATE $r_{\text{primary}}, r_{\text{secondary}} \gets r_{\text{secondary}}, r_{\text{primary}}$
    \STATE $\textit{is\_reconciling} \gets \texttt{false}$
\end{algorithmic}
\end{algorithm}

\begin{figure}
\centering
\includegraphics[width=0.48\textwidth]{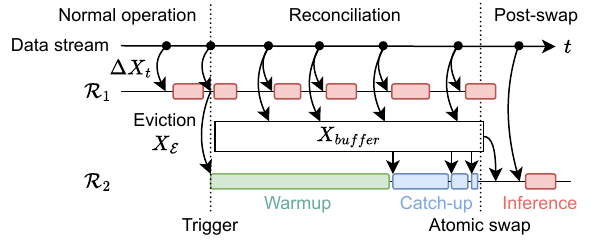}
\vspace{-4mm}
\caption{Asynchronous reconciliation process timeline. $\mathcal{R}_1$ serves real-time inference queries in parallel while $\mathcal{R}_2$ warms up on the evicted state $X_\mathcal{E}$ and integrates buffered chunks $X_{buffer}$ during catch-up. Upon synchronization, an atomic swap promotes $\mathcal{R}_2$ to primary, maintaining low-latency operation.}
\label{fig:reconciliation}
\vspace{-4mm}
\end{figure}

\noindent \textit{Prompt construction:} To emulate the CSR structure defined in Sec.~\ref{sec:csr_framework}, we construct input prompts $T$ as a concatenation of a static prefix $X_{static}$ and a dynamic suffix $X_{dyn}$:
\begin{equation}
    T = X_{static} \oplus X_{dyn}, \quad |X_{static}| = N, \;\; |X_{dyn}| = M.
\end{equation}
$X_{static}$ is sampled from random text with a fixed seed to ensure cache hits, while $X_{dyn}$ uses a random seed to enforce recomputation of the corresponding KV cache.

\noindent \textit{Metric:} We report TTFT, measuring the duration from request receipt to the generation of the first token. Each data point represents the mean and standard deviation over 5 independent trials.

\subsection{Evaluation Scenarios}

\noindent \textit{Latency analysis:}
We sweep the static context length $N \in [1, 120\text{K}]$ and dynamic suffix length $M \in \{1, \dots, 8192\}$. We compare CSR against a naive unordered baseline where the entire prompt $T$ is treated as dynamic (unordered), forcing a full cache re-computation at every step. This quantifies the overhead of $X_{task}$ size on real-time performance.


\noindent \textit{GPU comparison:}
To assess deployment feasibility, we evaluate CSR with $N \in \{30\text{K}, \dots, 120\text{K}\}$ across varying hardware configurations.

\noindent \textit{Infinite-horizon low-latency operation:}
We validate Algorithm~\ref{alg:asynchronous_reconciliation} by running a ROS 2 implementation on a physical robot system. The robot receives a continuous stream of state updates $\Delta X_t$ each containing 100 tokens. The robot runs on two \textit{Qwen 2.5-32B-Instruct} models; one model is used for the robot policy, and the other model is used for state reconciliation onset by a state eviction process. Eviction is triggered at a memory threshold $\tau_{mem}=110K$ tokens, and and discards the oldest half of the elements in $X_{chunks}$.
We measure full end-to-end LLM response cycle time and detect latency spikes by TTFT following how our robot policy predicts action decisions as single tokens. As our CSR framework reduce latency, not time for sequential output token generation, we refer to dedicated work on output generation speed for LLM-based robot policies for further application analysis.
See Appendix F for the pseudo-code describing the ASR implementation supporting concurrent requests and used on the physical robot system.

\noindent \textit{Embodied AI task performance:}
We evaluate CSR embodied AI task performance on AMA-bench~\cite{zhao2026amabench}. This benchmark measures long-horizon memory fidelity across continuous agent-environment interactions.
We compare performance against three competitive baselines representing the existing paradigms. First, StreamingLLM~\cite{xiao2023streamingllm} utilizing sliding attention windows representing kernel-level compute amortization. Secondly, BM25 representing a standard RAG implementation. Finally, AMA-Agent~\cite{zhao2026amabench}, a SOTA memory system leveraging causality graphs and tool-augmented retrieval.
We measure latency and jitter to characterize real-time inference capability of each method. Method performance is primarily estimated with the \textit{Qwen3.5-122B-A10B-FP8} models as the inference backbone. We also evaluate on the \textit{Qwen3-30B-A3B} model that is supported in TensorRT-LLM to serve StreamingLLM.


\section{RESULTS}
\label{sec:results}



\begin{table*}
\caption{GPU Performance Comparison (TTFT in seconds, mean)}
\vspace{-3mm}
\label{tab:gpu_performance_res}
\centering 
\begin{tabular}{|l|ccc|ccc|ccc|ccc|}
\hline
Input Length & \multicolumn{3}{c|}{30K} & \multicolumn{3}{c|}{60K} & \multicolumn{3}{c|}{90K} & \multicolumn{3}{c|}{120K} \\
M Value & M=1 & M=1024 & M=4096 & M=1 & M=1024 & M=4096 & M=1 & M=1024 & M=4096 & M=1 & M=1024 & M=4096 \\
\hline
4x H100 SXM5$^*$ & 0.13 & 0.20 & 0.41 & 0.22 & 0.31 & 0.65 & 0.30 & 0.42 & 0.89 & 0.38 & 0.56 & 1.13 \\
H100 SXM5 & 0.13 & 0.18 & 0.33 & 0.22 & 0.30 & 0.58 & 0.30 & 0.41 & 0.80 & 0.42 & 0.57 & 1.06 \\
GH200 & \textbf{0.11} & 0.16 & \textbf{0.31} & \textbf{0.19} & 0.27 & \textbf{0.54} & \textbf{0.27} & 0.37 & \textbf{0.76} & \textbf{0.36} & 0.49 & \textbf{0.99} \\
H100 PCIE & 0.29 & 0.32 & 0.61 & 0.45 & 0.51 & 1.02 & 0.68 & 0.71 & 1.45 & 0.78 & 0.88 & 1.87 \\
RTX PRO 6000 & 0.15 & \textbf{0.15} & 0.51 & 0.22 & \textbf{0.24} & 0.90 & 0.29 & \textbf{0.32} & 1.30 & 0.39 & \textbf{0.43} & 1.73 \\
\hline
\end{tabular}
\vspace{-4mm}
\end{table*}

\noindent \textit{Latency analysis:}
The primary bottleneck for real-time robotics is TTFT. As shown in Fig.~\ref{fig:res_latency} and Table~\ref{tab:latency_res}, CSR achieves a radical performance improvement over the unordered baseline. For a 120K token context, CSR reduces latency from $14.67$s to $0.56$s (26-fold reduction).
While latency scales linearly with dynamic suffix length $M$, CSR maintains sub-second response times $< 0.6$ s even with significant dynamic instructions $M=1024$. This indicates that the framework can accommodate substantial task-specific prompt engineering and large immediate observations without breaking the real-time loop.
These results provide direct empirical evidence that the optimal task structure formalized by CSR is strictly necessary to achieve manageable latency and avoid latency spikes. CSR thereby makes high-frequency control possible with massive LLMs on current general-purpose hardware.
%

\noindent \textit{GPU comparison:} Table~\ref{tab:gpu_performance_res} illustrate CSR's performance across different hardware. Notably, the workstation-class RTX PRO 6000 achieves the lowest latency for moderate dynamic loads ($M\approx1024$), outperforming data-center GPUs. This suggests that on-premise workstations are a viable and cost-effective platform for deploying massive embodied agents. However, for larger dynamic updates ($M \ge 4096$), the PCIe communication bandwidth likely becomes a bottleneck.

\noindent \textit{Infinite-horizon low-latency operation:}
Fig.~\ref{fig:ttft_vs_query_id} shows that CSR combined with the ASR algorithm (bottom) successfully eliminates the latency spike associated with an eviction operation in continuous operation over 30 minutes and 10 reconciliation cycles without introducing overhead.
Figure~\ref{fig:ttft_asr_summary} visualizes the latency distribution and jitter to characterize real-time predictability and validate our high-frequency control claims.
The density distribution (left) shows that TTFT remains bounded with a mean latency of 0.347 s and 99th percentile (p99) capped at 0.520 s.
The box plot (right)  demonstrates the long-horizon stability of the ASR process. The TTFT variance and median response lacks measurable drift across 10 successive state eviction cycles, backing our theoretical claims of perpetually stability and high-frequency operation.

\noindent \textit{Embodied AI task performance:}
Table~\ref{tab:ama_bench} shows the CSR framework successfully supports high-frequency operation without sacrificing extensive context required for complex embodied tasks.
The 30B model experiments demonstrate how StreamingLLM with a vanilla attention window size of 4096 tokens cannot represent long-horizon state memory. The maximum supported length of 32,768 tokens captures sufficient context but is slow due to having to recompute the KV cache for the large window at every state update.
For the more capable 122B model, CSR providing the complete state information achieves a SOTA recall of 0.836, substantially outperforming both the AMA agent with a graph-based memory system (0.459) and the BM25 RAG baseline (0.295) while maintaining a highly predictable, low-latency profile with minimal jitter (TTFT mean 0.220s, std 0.052s) similarly to the RAG method.




\begin{figure}
\centering
\includegraphics[width=0.49\textwidth]{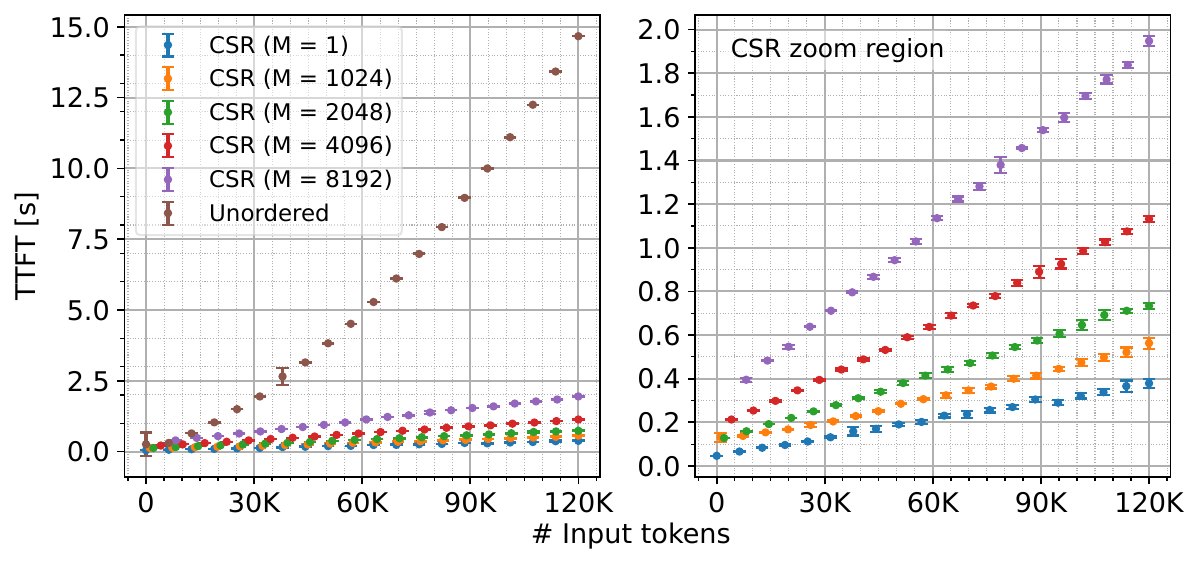}
\vspace{-8mm}
\caption{Latency experiment results. The left figure shows how prompts structured according to CSR radically reduces latency or time-to-first-token (TTFT) for long states and massive 235B model compared with generic unordered state prompts. The right figure display how larger dynamic state suffixes $x_{suff}$ and $x_{task}$ of token length M penalize latency.}
\label{fig:res_latency}
\vspace{-4mm}
\end{figure}

\begin{table}[h]
\vspace{2mm}
\caption{Latency results (mean ± std in seconds)}
\label{tab:latency_res}
\vspace{-3mm}
\begin{tabular}{|c|c|c|c|c|}
\hline
 & 30K & 60K & 90K & 120K \\
\hline
Unord. & 1.83 ± 0.00 & 4.90 ± 0.01 & 9.22 ± 0.01 & 14.67 ± 0.01 \\ \hline
M=1 & \textbf{0.12 ± 0.00} & \textbf{0.21 ± 0.02} & \textbf{0.27 ± 0.01} & \textbf{0.34 ± 0.01} \\ \hline
M=1024 & 0.20 ± 0.01 & 0.31 ± 0.01 & 0.42 ± 0.01 & 0.56 ± 0.03 \\ \hline
M=2048 & 0.26 ± 0.00 & 0.42 ± 0.01 & 0.58 ± 0.01 & 0.73 ± 0.01 \\ \hline
M=4096 & 0.41 ± 0.00 & 0.65 ± 0.01 & 0.89 ± 0.03 & 1.13 ± 0.02 \\ \hline
M=8192 & 0.69 ± 0.00 & 1.11 ± 0.01 & 1.53 ± 0.01 & 1.95 ± 0.02 \\ \hline
\end{tabular}
\end{table}


\begin{figure}[t!]
\centering
\includegraphics[width=0.485\textwidth]{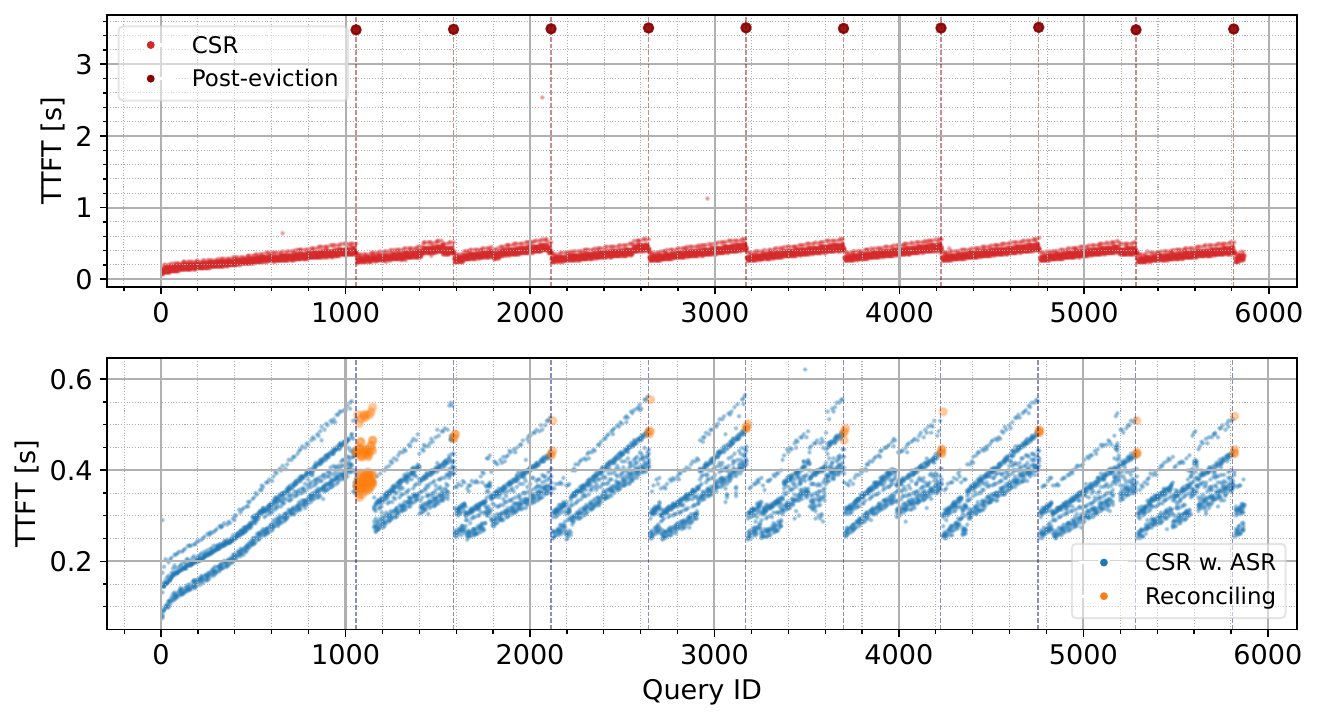}
\vspace{-8mm}
\caption{CSR with ASR (bottom) avoids latency spikes after a state eviction operation and infinite horizon low-latency operation.}
\label{fig:ttft_vs_query_id}
\vspace{-3mm}
\end{figure}

\begin{figure}[t!]
\centering
\includegraphics[width=0.485\textwidth]{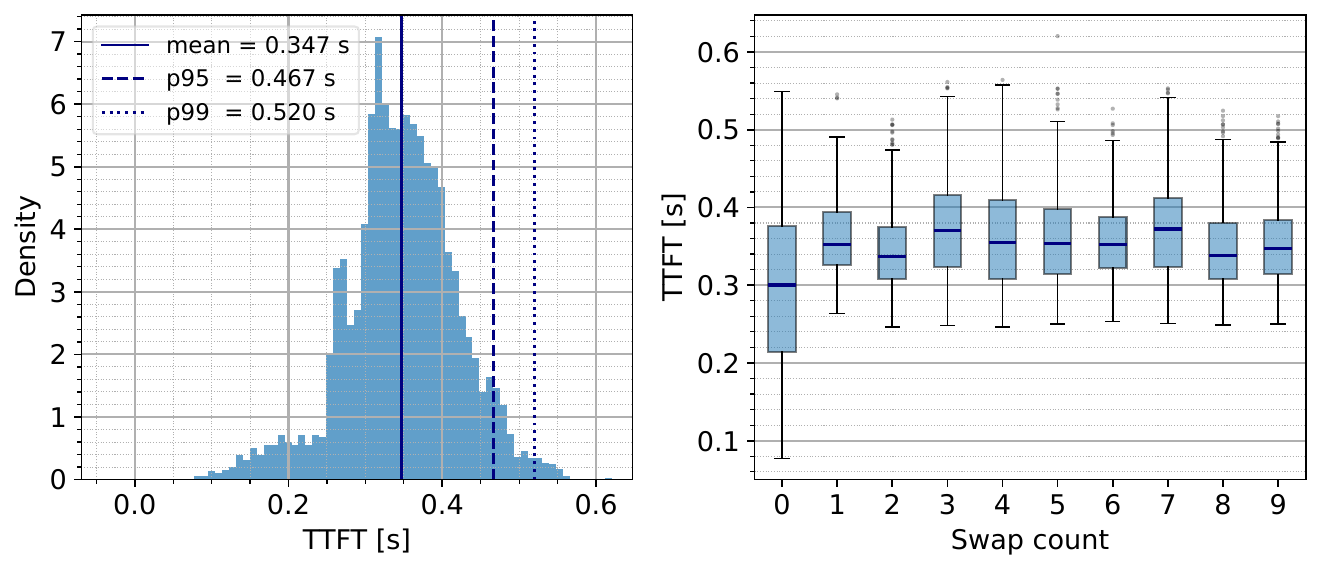}
\vspace{-8mm}
\caption{CSR with ASR maintains a bounded (left) and consistent (right) TTFT distribution over 10 eviction cycles.}
\label{fig:ttft_asr_summary}
\vspace{-3mm}
\end{figure}

\begin{table}[h]
\vspace{2mm}
\caption{AMA-Bench: Embodied AI Domain}
\label{tab:ama_bench}
\vspace{-3mm}
\centering
\begin{tabular}{|l|l|c|c|}
\hline
\multirow{2}{*}{Model} & \multicolumn{1}{c|}{\multirow{2}{*}{Method}} & \multicolumn{2}{c|}{Embodied AI} \\ \cline{3-4}
 & & TTFT mean (std) [s] & Recall \\ \hline
\multirow{5}{*}{30B} & StreamingLLM (4096$^\ast$)~\cite{xiao2023streamingllm} & 0.228 (0.0302) & 0.000 \\ \cline{2-4}
 & StreamingLLM (32,768$^\ast$)~\cite{xiao2023streamingllm} & 2.674 (1.084) & 0.180 \\ \cline{2-4}
 & BM25 & \textbf{0.137 (0.0402)} & 0.066 \\ \cline{2-4}
 & AMA-Agent~\cite{zhao2026amabench} & 2.255 (1.960) & 0.066 \\ \cline{2-4}
 & CSR (ours) & 0.179 (0.0460) & 0.115 \\ \hline
\multirow{3}{*}{122B} & BM25 & 0.2059 (0.1027) & 0.295 \\ \cline{2-4}
 & AMA-Agent~\cite{zhao2026amabench} & 0.9262 (0.287) & 0.459 \\ \cline{2-4}
 & CSR (ours) & 0.220 (0.052) & \textbf{0.836} \\ \hline
 \multicolumn{4}{l}{\footnotesize $^\ast$: Sliding attention window size [tokens] } \\
\end{tabular}
\end{table}

\section{CONCLUSIONS}
\label{sec:conclusions}

We presented a theoretical and empirical investigation of state structure and the latency bottleneck in LLM-based embodied intelligence. By formalizing the sequential robot AI inference problem, we proved that Prefix Stability, Incremental Extensibility, and Asynchronous State Reconciliation are necessary conditions for continuous real-time operation even with massive models.
The CSR framework instantiates these conditions as a practical task structuring strategy, enabling optimal KV-cache reuse across sequential inference tasks without modifying the underlying model or inference engine.
To sustain these properties over infinite horizons, the Asynchronous State Reconciliation (ASR) algorithm offloads state eviction to a secondary compute resource, eliminating the latency spikes that would otherwise periodically violate real-time deadlines.
Empirical validation on a physical robot with a 235B parameter model confirms a 26-fold TTFT reduction at 120K token contexts, spike-free operation across 10 eviction cycles, and SOTA embodied task recall of 0.836, nearly double the next best system by retaining the complete state context while maintaining RAG-level latency.
Together, these results demonstrate that wirelessly connected on-premise servers running massive cognitive engines are a viable architecture for high-frequency embodied AI, and that the path to capable real-time agents runs through principled state representation rather than model compression or context truncation.

\section*{APPENDIX}

\subsection{Proof for Theorem~\ref{theorem:prefix_stability}}
\label{proof:theorem_1}

\begin{proof} The total number of floating-point operations $N_{FLOPs}$ to complete a task $T_t$ with the real-time constraint \eqref{eq:real-time_constraint} requires
\begin{equation}
    \label{eq:flops_constraint}
    N_{FLOPS}(T_t) \le \tau_{max} F_{hw},
\end{equation}
where $F_{hw}$ is the hardware's FLOPS. We define
\begin{equation}
    \label{eq:max_comp_budget}
    B_{max} = \tau_{max} F_{hw}
\end{equation}
as the maximum computational budget corresponding to~\eqref{eq:flops_constraint}. We construct a proof by contradiction by assuming prefix stability is not necessary to satisfy \eqref{eq:flops_constraint} and \eqref{eq:max_comp_budget} with $i^* = 1$ in \eqref{eq:first_differing_token}:
\begin{equation}
    \label{eq:b_max_assumption}
    N_{FLOPs}(T_t^{i^*=1}) \le B_{max}
\end{equation}

The number of FLOPs is proportional to the number of tokens estimated in \eqref{eq:TTFT_derivation_2}:
\begin{equation}
    \label{eq:flops_i_1}
    N_{FLOPs}(T_t^{i^*=1}) \propto \tfrac{1}{2} L \: d_h \: |T_t|^2.
\end{equation}
Estimating the number of FLOPs for sequential tasks with prefix stability with a small token addition $\Delta \ll |T_t|$ gives
\begin{equation}
    \label{eq:flops_i_delta}
    N_{FLOPs}(T_t^{i^*=|T_t| - \Delta}) \propto \tfrac{1}{2} L \; d_h \; \Delta |T_t|
\end{equation}
The ratio of \eqref{eq:flops_i_1} and \eqref{eq:flops_i_delta} is
\begin{equation}
    \label{eq:flops_ratio}
    \frac{N_{FLOPs}(T_t^{i^*=1})}{N_{FLOPs}(T_t^{i^*=|T_t| - \Delta})} \approx \frac{|T_t|}{\Delta}.
\end{equation}

In this work, we consider $|T_t| \sim \mathcal{O}(10^5)$ tokens per task and small token increments $\Delta \sim \mathcal{O}(10^3)$ . Accordingly, \eqref{eq:flops_ratio} estimates
\begin{equation}
    \label{flops_ratio_w_tokens}
    \frac{N_{FLOPs}(T_t^{i^*=1})}{N_{FLOPs}(T_t^{i^*=|T_t| - \Delta})} \approx \frac{10^5}{10^3} = 10^2.
\end{equation}
Latency for LLMs with very large input token sequences is a recognized problem and experimentally verified in Fig.~\ref{fig:res_latency}, and lower latency is advantageous for a real-time robot AI system. Suppose it is possible to achieve satisfactory high-frequency operation by utilizing prefix stability such that
\begin{equation}
    \label{eq:prefix_stability_constraint_satisfaction}
    N_{FLOPs}(T_t^{i^*=|T_t| - \Delta}) \le B_{max}.
\end{equation}
However, if latency without stable prefix $N_{FLOPs}(T_t^{i^*=1})$ is $\mathcal{O}(100)$ times higher than when using stable prefix, it is a computational certainty that
\begin{equation}
    N_{FLOPs}(T_t^{i^*=1}) \gg B_{max}.
\end{equation}
Therefore, the assumption \eqref{eq:b_max_assumption} is false, meaning prefix stability is necessary to satisfy the real-time constraint \eqref{eq:prefix_stability_constraint_satisfaction} for large input token sequences.
\end{proof}

\subsection{Proof for Theorem~\ref{theorem:incremental_extensibility}}
\label{proof:theorem_2}

\begin{proof}
Let $S$ be a task structuring strategy that satisfies constraints \eqref{eq:real-time_constraint} and \eqref{eq:information_sufficiency_constraint}.
Theorem~\ref{theorem:prefix_stability} determines that $S$ must ensure prefix stability to satisfy \eqref{eq:real-time_constraint}. $S$ implementing the concatenation operation $\oplus$ in \eqref{eq:incremental_extensibility} ensures optimal prefix stability by achieving $i^* = |T_t| + 1$, and is the only optimal structure.

$S$ satisfies the information sufficiency constraint \eqref{eq:information_sufficiency_constraint} based on the premise that the total information context is non-decreasing
\begin{equation}
    \text{Info}(T_{t+1} | S) \supseteq \text{Info}(T_t | S) + \text{Info}(\Delta X_t).
\end{equation}
This ensures that any arbitrary task requirement can be met:
\begin{equation}
    \mathcal{I}_{req}^{(t+1)} \subseteq \text{Info}(T_{t+1} | S).
\end{equation}

Next, we prove that the necessity of Theorem~\ref{theorem:incremental_extensibility} by contradiction.
Assume that $S$ is not incrementally extensible. This means $S$ does not adhere to the structure in \eqref{eq:incremental_extensibility}.
Thus $S$ must modify $T_t$ in a way other than appending (e.g., insertion, deletion, or replacement) at an index $i \le |T_t|$ to encode the required information $\text{Info}(T_{t+1} | S)$. By the definition of $i^*$ in \eqref{eq:first_differing_token}, any such modification forces $i^*(T_t, T_{t+1}) \le |T_t|$. Smaller $i^*$ values increasingly violates the condition for sufficient prefix stability specified in Theorem~\ref{theorem:prefix_stability}.

This result contradicts the assumption that $S$ need not be incrementally extensible as defined in \eqref{eq:incremental_extensibility}.
\end{proof}

\subsection{Proof for Theorem~\ref{theorem:eviction}}
\label{proof:theorem_3}

\begin{proof}
We prove the necessity of eviction by contradiction.
Assume a strategy $S$ exists for infinite-horizon operation that satisfies the constraints \eqref{eq:real-time_constraint}-\eqref{eq:finite_memory_constraint} without an eviction operation $\mathcal{E}$.
Satisfying \eqref{eq:real-time_constraint} and \eqref{eq:information_sufficiency_constraint} means $S$ must be incrementally extensible according to Theorem~\ref{theorem:incremental_extensibility}.
The length of task sequence $|T_t|$ is therefore monotonically non-decreasing:
\begin{equation}
    |T_{t+1}| = |T_t| + \Delta X_t \ge |T_t|.
\end{equation}
For any non-trivial robot AI system, new information is continuously processed:
\begin{equation}
    \mathbb{E}[|\Delta X_t|] > 0.
\end{equation}
The total sequence length thus grows without bound:
\begin{equation}
    \label{eq:infinite_seq_len}
|T_t| = |T_0| + \sum_{\tau=0}^{t-1} |\Delta X_\tau| \implies \lim_{t \to \infty} |T_t| = \infty.
\end{equation}
Assuming memory cost is proportional to sequence length: $|T_t| \propto \text{Memory}(T_t|S)$), there will exist a time $t_{OOM}$ such that for all $t \ge t_{OOM}$ the finite memory constraint \eqref{eq:finite_memory_constraint} is violated.

Therefore, the initial assumption is false. Any viable $S$ for infinite-horizon operation must include an operation $\mathcal{E}$ as in \eqref{eq:eviction_operation} to prevent unbounded growth by removing information.
\end{proof}

\subsection{Proof for Theorem~\ref{theorem:reconciliation}}
\label{proof:theorem_4}

\begin{proof}
We prove that the eviction processes must be offloaded by contradiction.
Suppose a task representation $T_t$ satisfies Theorem~\ref{theorem:prefix_stability} with $i^* \approx |T_t|$.
Next an eviction operation \eqref{eq:eviction_operation} is performed synchronously on $T_t$.
Any non-trivial eviction operation necessarily modifies the prefix structure so that $i^* \ll |T_t|$, and thus violate \eqref{eq:prefix_stability_condition} in Theorem~\ref{theorem:prefix_stability}, meaning real-time operation is not possible.
Therefore, synchronous in-place eviction is not feasible and the operation must be offloaded.

Next, we create a proof by construction for the asynchronous state reconciliation processes:
\begin{enumerate}
    \item Duplicate $T_t \rightarrow \hat{T}_t$ when $|T_t|$ reaches a memory threshold $\tau_{mem} < N_{max}$ at time $t_{evict}$.
    \item Apply the eviction operation on the copy:
    \begin{equation}
        T_t^{\mathcal{E}} = \mathcal{E}(\hat{T}_t).
    \end{equation}
    \item Warm up $T_t^{\mathcal{E}}$ by recomputing the KV cache on the secondary hardware resource $\mathcal{R}_2$.
    \item The primary hardware resource $\mathcal{R}_1$ continue to process tasks and append new tokens $\Delta X_{t_{evict}:T}$ accumulated during the warm-up.
    \item Once $T_t^{\mathcal{E}}$ has completed warm-up and is synced or reconciled:
    \begin{equation}
        T_T^{\mathcal{E}} \equiv \mathcal{E}(T_t) \oplus \Delta X_{t_{evict}:T},
    \end{equation}
    the pointer for $T_T$ in $\mathcal{R}_1$ is swapped to $T_T^{\mathcal{E}}$ in $\mathcal{R}_2$:
    \begin{equation}
        T_T := T_T^{\mathcal{E}}.
    \end{equation}
\end{enumerate}
The new sequential task $T_{T+1}$ is now an incremental extension of the evicted task $\mathcal{E}(T_t)$ loaded in $\mathcal{R}_2$.

The processes is viable if and only if the total time for the parallel work (e.g., warm-up and reconciliation) is less than the time to memory overflow failure on $\mathcal{R}_1$:
\begin{equation}
    \label{eq:recon_time_inequlity}
    \Delta t_{warm} + \Delta t_{recon} \le \Delta t_{mem}.
\end{equation}

The worst-case warm-up time $\Delta t_{warm}$ represent the cost of recomputing the new prefix $T_T^{\mathcal{E}}$ from scratch ($i^* = 1$). From \eqref{eq:TTFT_derivation_2}:
\begin{equation}
    \label{eq:warm-up_time}
    \Delta t_{warm} \propto |T_t^{\mathcal{E}}|^2 = (\epsilon \cdot |T_t|)^2
\end{equation}
where $\epsilon$ is the ratio of tokens remaining after eviction $\mathcal{E}(T_t)$.

Reconciliation time $\Delta t_{recon}$ represents the cost of extending $T_t^{\mathcal{E}}$ by the tokens accumulated during warm-up:
\begin{equation}
    \Delta N = \dot{X} \cdot \Delta t_{warm},
\end{equation}
where $\dot{X}$ is the rate new tokens are appended to $T_t$ on $\mathcal{R}_1$.
The time to extend $T_t^{\mathcal{E}}$ is given by \eqref{eq:TTFT_derivation_2} with $i^*=|T_t^{\mathcal{E}}|+1$:
\begin{multline}
    \label{eq:recon_time}
    \Delta t_{recon} \propto \sum_{i=|T_t^{\mathcal{E}}|+1}^{|T_t^{\mathcal{E}}| + \Delta N} i \\
    = \frac{(|T_t^{\mathcal{E}}| + \Delta N - (|T_t^{\mathcal{E}}|+1) + 1)(|T_t^{\mathcal{E}}| + \Delta N + (|T_t^{\mathcal{E}}|+1))}{2} \\
    = \frac{\Delta N (2 |T_t^{\mathcal{E}}| + \Delta N + 1)}{2} \propto \Delta N (2 |T_t^{\mathcal{E}}| + \Delta N + 1).
\end{multline}

The time to out of memory failiure $\Delta t_{mem}$ once eviction is triggered is
\begin{equation}
    \label{eq:oom_time}
    \Delta t_{mem} = \frac{N_{buffer}}{\dot{X}} = \frac{N_{max} - |T_t|}{\dot{X}}
\end{equation}
with $N_{buffer}$ representing the number of free space tokens.

Substituting \eqref{eq:warm-up_time}, \eqref{eq:recon_time}, and \eqref{eq:oom_time} into \eqref{eq:recon_time_inequlity}:
\begin{equation}
    (\epsilon \cdot |T_t|)^2 + \Delta N (2 |T_t^{\mathcal{E}}| + \Delta N + 1) \le N_{buffer} / \dot{X}
\end{equation}
Since we can always set the free variables $\epsilon$ and $N_{buffer}$ such that the inequality holds, the construction is proved to be a viable and necessary method to satisfy all three theorems. 
\end{proof}

\subsection{Proof for Corollary~\ref{corollary:optimality_of_csr}}
\label{proof:corollary_1}

\begin{proof}
Theorem~\ref{theorem:prefix_stability} states that the primary objective~\eqref{eq:seq_robot_ai_optimization_obj} is optimized when strategy $S$ results in optimal prefix stability~\eqref{eq:optimal_strategy_definition}.
The real-time constraint~\eqref{eq:real-time_constraint} must be satisfied by the optimal prefix stability, or no valid strategy exists (e.g., $\tau_{max}$ is set too low for given hardware and sequence length $|T_n|$).
Continuous real-time operation is possible for $S$ implementing asynchronous state reconciliation as defined in Theorem~\ref{theorem:reconciliation}.

Theorem~\ref{theorem:incremental_extensibility} states that the information constraint~\eqref{eq:information_sufficiency_constraint} is satisfied for any $S$ that is incrementally extensible~\eqref{eq:incremental_extensibility}.

A strategy $S$ providing the eviction~\eqref{eq:eviction_operation} in Theorem~\ref{theorem:eviction} together with the reconciliation method in Theorem~\ref{theorem:reconciliation} satisfies the finite memory constrain~\eqref{eq:finite_memory_constraint}.

As~\eqref{eq:seq_robot_ai_optimization_obj} is optimized and all constraints~\eqref{eq:real-time_constraint}-\eqref{eq:finite_memory_constraint} are satisfied, the strategy $S$ must be the optimal strategy $S^*$.
\end{proof}

\subsection{Concurrent Request Routing ASR Algorithm}

Algorithm~\ref{alg:concurrent_req_asr} provides the complete per-query inference routing procedure as implemented in the real-world experiments, extending Algorithm~\ref{alg:asynchronous_reconciliation} to the concurrent request setting. While Algorithm~\ref{alg:asynchronous_reconciliation} describes the background state management logic (buffering, warm-up, catch-up, and the atomic swap), it abstracts over the question of how individual inference queries are served during the reconciliation window. Algorithm~\ref{alg:concurrent_req_asr} addresses this by introducing explicit sequence versioning (e.g. $k$, $k_t$) to classify each incoming query into one of three routing cases: (1) normal continuation on the primary resource, (2) reconciliation-phase bridging where the query is served on the primary resource using a stitched context $X_recon$ constructed from the current sequence extended with chunks received since eviction, and (3) straggler routing where delayed queries referencing the previous sequence version are served on the secondary resource, which retains its intact KV cache until the next eviction cycle.

\begin{strip}
\par\vspace{1em}\noindent 
\hrule height 0.8pt 
\vspace{3pt}
\refstepcounter{algorithm} 
\noindent \textbf{Algorithm \thealgorithm} Asynchronous State Reconciliation with Concurrent Request Routing
\label{alg:concurrent_req_asr}
\vspace{3pt}
\hrule height 0.4pt 
\vspace{3pt}
\begin{algorithmic}[1]
\REQUIRE Stream of states $X_t$, reconciliation catch-up threshold $N_{\text{catch-up}}$
\item[\textbf{State Variables:}]
\item[] $j \gets 0$  \hfill // \textit{Current active sequence static state char length}
\item[] $k \gets 0$  \hfill // \textit{Current active sequence version; current sequence is $k$, new evicted sequence is $k+1$}
\item[] $k_\text{ready} \gets 0$ \hfill // \textit{Highest sequence version fully precomputed on secondary resource}
\item[] $j_\varepsilon \gets \emptyset$ \hfill // \textit{Evicted sequence static state char length}

\item[] $X_{\text{recon}} \gets \emptyset$  \hfill // \textit{Reconciliation state: Extends current sequence with chunks from new evicted sequence}
\item[] $\Delta X_\varepsilon \gets \emptyset$ \hfill // \textit{Evicted state chunk buffer used during catch-up reconciliation phase} 
\item[] $r_{\text{primary}} \gets \mathcal{R}_1$ \hfill // \textit{Currently active resource with KV cache computed for current sequence}
\item[] $r_{\text{secondary}} \gets \mathcal{R}_2$ \hfill // \textit{Resource used to pre-compute the KV cache for the new evicted sequence}
\item[] $\text{is\_reconciling} \gets \texttt{false}$ \hfill // \textit{Status flag}
\item[] $\mathcal{L}_{\Delta X_\varepsilon} \gets \texttt{Mutex()}$ \hfill // \textit{Lock for the catch-up buffer}
\STATE \textbf{Procedure} \textsc{Inference}($X_t$):
    \STATE $j_t, k_t, j_{\varepsilon, t} \gets X_t$ \hfill // \textit{Unpack state metadata}
    \IF{$j_\varepsilon = \emptyset$}
        \STATE $j_\varepsilon = j_{\varepsilon,t}$ \hfill // \textit{State char length of old chunks in evicted states (e.g. excluding new chunks)}
    \ENDIF
    \STATE // \textit{Case 1: Sequence continuation}
    \IF{$k_t = k$}
        \IF{$j_t > j$}
            \STATE $X_{\text{recon}} \gets X_t[:j_t]$ \hfill // \textit{For use during reconciliation}
            \STATE $j \gets j_t$
        \ENDIF
        \STATE $\text{LLM}(X_t, r_{\text{primary}})$
    \ENDIF
    \STATE // \textit{Case 2: Reconciliation}
    \IF{$k_t > k$}
        \IF{$k_t > k_\text{ready}$}
            \STATE // \textit{$r_\text{secondary}$ has not finished precomputation $\rightarrow$ Bridge state and run on $r_\text{primary}$}
            \STATE $\Delta X_t \gets X_t[j_\varepsilon:j_t]$ \hfill // \textit{Extract new chunks from evicted state}
            \STATE $X_{\text{recon}} \gets X_{\text{recon}} \oplus \Delta X_t$ \hfill // \textit{Append chunks to reconciliation state }
            \STATE $\text{Lock}(\mathcal{L}_{\Delta X_\varepsilon})$
            \STATE $\Delta X_\varepsilon \gets \Delta X_\varepsilon \oplus \Delta X_t$ \hfill // \textit{Append chunks to catch-up buffer}
            \STATE $\text{Unlock}(\mathcal{L}_{\Delta X_\varepsilon})$
            \STATE $X_\text{dyn} \gets X_t[j_t:]$ \hfill // \textit{Extract dynamic suffix (incl. task instruction)}
            \STATE $X'_t \gets X_{\text{recon}} \oplus X_\text{dyn}$ \hfill // \textit{Construct task query based on extended current state sequence}
            \STATE $j_\varepsilon \gets j_t$ \hfill // \textit{Update eviction cursor}
            \STATE $\text{LLM}(X'_t, r_{\text{primary}})$ 
        \ELSIF{$k_t = k_\text{ready}$}
            \STATE // \textit{$r_\text{secondary}$ has finished precomputation $\rightarrow$ Resource swap}
            \STATE $r_{\text{primary}}, r_{\text{secondary}} \gets r_{\text{secondary}}, r_{\text{primary}}$ \hfill // \textit{Swap primary and secondary resource pointers}
            \STATE $\text{is\_reconciling} \gets \texttt{false}$
            \STATE $X_{\text{recon}} \gets X_t[:j_t]$ \hfill // \textit{For use during straggler queries}
            \STATE $j, k, j_\varepsilon \gets j_t, k_t, \emptyset$ \hfill // \textit{Update state cursors and run inference on new sequence}
            \STATE $\text{LLM}(X_t, r_\text{primary})$
        \ENDIF
    \ENDIF
    \STATE // \textit{Case 3: Straggler queries}
    \IF{$k_t < k$}
        \IF{$\text{is\_reconciling}$}
            \STATE // \textit{Protect $r_\text{secondary}$ from significantly delayed requests}
            \STATE $X_\text{dyn} \gets X_t[j_t:]$
            \STATE $X'_t \gets X_{\text{recon}} \oplus X_\text{dyn}$
            \STATE $\text{LLM}(X'_t, r_{\text{primary}})$
        \ELSE
            \STATE $\text{LLM}(X_t, r_\text{secondary})$ \hfill // \textit{Secondary resource holds the intact KV cache for the old sequence}
        \ENDIF
    \ENDIF
    
    \STATE

    \STATE \textbf{Procedure} \textsc{Reconcile}($X_\varepsilon, k_\text{target}$):
    \STATE $\text{is\_reconciling} \gets \texttt{true}$ \hfill // \textit{Lock $r_\text{secondary}$ from stragglers}
    \STATE $\text{Lock}(\mathcal{L}_{\Delta X_\varepsilon})$
    \STATE $\Delta X_\varepsilon \gets \emptyset$
    \STATE $\text{Unlock}(\mathcal{L}_{\Delta X_\varepsilon})$
    \STATE $\text{LLM}(X_\varepsilon, r_\text{secondary})$ \hfill // \textit{Precompute KV cache on secondary resource with evicted state}
    \STATE
    \STATE // \textit{Reconcile buffered chunks until sufficiently caught up}
    \STATE // \textit{NOTE: Reminder is computed at next task inference}
    \WHILE{$|\Delta X_\varepsilon | > N_{\text{catch-up}}$}
        \STATE $\text{Lock}(\mathcal{L}_{\Delta X_\varepsilon})$
        \STATE $\Delta X'_\varepsilon \gets \Delta X_\varepsilon$ \hfill // \textit{Copy and empty buffer for appending new chunks while precomputing KV cache}
        \STATE $\Delta X_\varepsilon \gets \emptyset$
        \STATE $\text{Unlock}(\mathcal{L}_{\Delta X_\varepsilon})$
        \STATE $X_{\varepsilon} \gets X_{\varepsilon} \oplus \Delta X'_\varepsilon$
        \STATE $\text{LLM}(X_{\varepsilon}, r_{\text{secondary}})$
    \ENDWHILE
    \STATE $\text{Lock}(\mathcal{L}_{\Delta X_\varepsilon})$
    \STATE $\Delta X_\varepsilon \gets \emptyset$
    \STATE $\text{Unlock}(\mathcal{L}_{\Delta X_\varepsilon})$
    \STATE $k_\text{ready} \gets k_\text{target}$ \hfill // \textit{Authorize the swap}
    
\end{algorithmic}
\vspace{3pt}
\hrule height 0.4pt 
\vspace{1em}
\end{strip}

\section*{ACKNOWLEDGMENT}

Robin Karlsson was partially supported by grant JPNP23003, commissioned by the New Energy and Industrial Technology Development Organization (NEDO).

\addtolength{\textheight}{-12cm}   

\bibliographystyle{IEEEtran}
\bibliography{references}

\end{document}